%% file: main.tex
\ifdefined\pdftexversion\pdfoutput=1\fi
\documentclass[10pt]{article}

\usepackage{iftex}
\ifPDFTeX
  \usepackage[T1]{fontenc}
\fi
\usepackage[letterpaper,margin=1in]{geometry}
\usepackage{libertine}
\usepackage{amsthm}
\usepackage[libertine]{newtxmath}
\ifXeTeX
  \DeclareSymbolFont{operators}{OT1}{LinuxLibertineT-TLF}{m}{n}
  \SetSymbolFont{operators}{bold}{OT1}{LinuxLibertineT-TLF}{b}{n}
\fi
\usepackage{microtype}
\usepackage{graphicx}
\usepackage{booktabs}
\usepackage{siunitx}
\usepackage{tabularx}
\usepackage{xltabular}
\usepackage{multirow}
\usepackage{makecell}
\usepackage{amsmath}
\usepackage{xcolor}
\usepackage{tikz}
\usetikzlibrary{arrows.meta,positioning,fit,backgrounds,calc,shapes.geometric}
\usepackage{enumitem}
\usepackage{listings}
\usepackage{float}
\usepackage{placeins}
\usepackage{algorithm}
\usepackage[noend]{algpseudocode}
\usepackage[labelfont=bf,font=small]{caption}
\usepackage[numbers,square,sort&compress]{natbib}
\usepackage[colorlinks=true,linkcolor=blue!55!black,citecolor=blue!55!black,urlcolor=blue!55!black]{hyperref}
\usepackage[capitalise,noabbrev]{cleveref}
\hypersetup{pdftitle={Mnemon: Raw Records, Fast Judgments, Slow Thoughts}, pdfauthor={Guangren Wang}}
\newtheorem{proposition}{Proposition}
\crefname{proposition}{Proposition}{Propositions}

\definecolor{ours}{HTML}{2a78d6}
\definecolor{other}{HTML}{eb6834}
\definecolor{ink2}{HTML}{52514e}
\definecolor{rule}{HTML}{c3c2b7}
\definecolor{fillA}{HTML}{e9f1fb}
\definecolor{fillB}{HTML}{fdeee7}
\definecolor{fillC}{HTML}{e6f6ef}

\newcommand{\sys}{Mnemon}

\newcommand{\jev}{Jev}
\newcommand{\lme}{Long\-Mem\-Eval-S}

\setlist{itemsep=1pt,topsep=3pt}
\algrenewcommand\algorithmicrequire{\textbf{Input:}}
\crefname{algorithm}{Algorithm}{Algorithms}
\crefname{equation}{Equation}{Equations}
\crefformat{equation}{Equation~#2#1#3}

\input{tables/numbers}
\newenvironment{gentable}[1][t]{\begin{table}[#1]}{\end{table}}

\title{\sys: Raw Records, Fast Judgments, Slow Thoughts}
\author{Guangren Wang}
\date{}

\begin{document}
\maketitle

\input{sections/abstract}
\input{sections/introduction}
\input{sections/related}
\input{sections/setting}
\input{sections/system}
\input{sections/implementation}
\input{sections/evaluation}
\input{sections/protocol}
\input{sections/results}
\input{sections/analysis}
\input{sections/limitations}
\input{sections/conclusion}

\bibliographystyle{plainnat}
\bibliography{references}


\end{document}

%% file: tables/numbers.tex
\newcommand{\haluSystems}{13}

\newcommand{\beamSystems}{12}
\newcommand{\finLoCoMo}{91.7}

\newcommand{\finCtxLoCoMo}{3.8k}

\newcommand{\consLoCoMoAbs}{2.5}

\newcommand{\writeLoCoMo}{0.013}
\newcommand{\dsFinLoCoMo}{92.2}
\newcommand{\dsFinLoCoMoMini}{92.8}

\newcommand{\finEciLoCoMo}{0.259}
\newcommand{\eciOtherLoCoMo}{0.337}
\newcommand{\eciOtherNameLoCoMo}{mem9}
\newcommand{\dsConsLoCoMoAbs}{1.9}
\newcommand{\dsCostLoCoMo}{2.38}
\newcommand{\finLME}{83.8}

\newcommand{\consLMEAbs}{4.4}

\newcommand{\dsFinLME}{94.4}
\newcommand{\dsFinLMEMini}{93.4}

\newcommand{\finEciLME}{0.198}
\newcommand{\eciOtherLME}{0.147}
\newcommand{\eciOtherNameLME}{MemOS}
\newcommand{\dsConsLMEAbs}{4.0}
\newcommand{\dsCostLME}{2.53}
\newcommand{\dsFinLoCoMoRev}{95.3}
\newcommand{\jevTokRange}{35--73k}
\newcommand{\jevRatioRange}{9--19}
\newcommand{\readerGainLME}{9.6}
\newcommand{\readerGainLMECI}{+6.3 to +12.9}

\newcommand{\consBeamAbs}{6.2}

\newcommand{\consCostRange}{1.17--1.37}

\newcommand{\finHaluRankWord}{eighth}
\newcommand{\finBeamRankWord}{tenth}

\newcommand{\beamTenC}{51.2}

\newcommand{\beamTenCostRatioC}{1.11}

\newcommand{\consBeamTenAbs}{5.9}
\newcommand{\beamTenBatches}{6,095}

\newcommand{\workJevCalls}{5--10}
\newcommand{\workJevWaves}{4--7}

\newcommand{\workReplan}{6--31}

\newcommand{\workJevTime}{1.4--2.4}
\newcommand{\workSearchSmall}{14--18}
\newcommand{\workReadSmall}{0.1--0.2}
\newcommand{\workSearchTen}{248}
\newcommand{\workReadTen}{3.5}
\newcommand{\workTenRecords}{108,810}
\newcommand{\workWave}{0.34}
\newcommand{\workMeasured}{10--15}
\newcommand{\jmAcc}{84.4}
\newcommand{\jmAccDS}{82.1}
\newcommand{\jmDiff}{7.3}
\newcommand{\jmCI}{+5.5 to +9.2}
\newcommand{\jmUp}{166}
\newcommand{\jmDown}{53}

\newcommand{\jmMultiHop}{14.2}
\newcommand{\jmTemporal}{8.4}
\newcommand{\jmCtx}{2.6k}
\newcommand{\jmEci}{0.275}
\newcommand{\jmLabels}{84.0}
\newcommand{\jmWrite}{0.12}
\newcommand{\jmWriteRatio}{9.5}

\newcommand{\jmJevCalls}{4.1}

%% file: sections/abstract.tex
\begin{abstract}
Long-term memory lets an LLM assistant use a history it can no longer reread, and most memory systems build it by rewriting conversations into facts, graphs or typed memories at write time.
We argue that the work of memory divides, as thinking does, into two systems.
Most of it is fast System~1 work: many small, independent yes/no judgments about records, such as whether a record is needed or no longer current, which a decision model makes by the dozen in a third of a second.
Only a little is slow System~2 work: writing a few search queries, naming what the reply needs and composing the answer, which an LLM does well but slowly.
We present \sys, a memory agent built on this division.
It keeps conversations as raw, dated records; an LLM (System~2) plans searches over them, a decision model, \jev{} (System~1), judges what the searches return, and rules with explicit budgets turn the judgments into a small View for an unchanged answering model.
A background pass consolidates each record once into topic timelines, value histories and standing instructions linked to the records, so that questions about a whole conversation reach evidence their own searches miss.
Because nothing is decided about a record when it is written, the same agent can read any store that returns dated records.

With gpt-4.1-mini answering, as in a public re-evaluation of 14 systems, \sys{} scores \finLoCoMo\% on LoCoMo, the highest among them, and \finLME\% on \lme, from under 4k tokens of context per question, with the lowest effective cost index on LoCoMo.
With a reasoning model answering, it reaches \dsFinLoCoMo\% on LoCoMo and \dsFinLME\% on \lme, the latter on par with the best published results.
From 100K to 10M tokens of history on BEAM, its cost per question grows by a factor of \beamTenCostRatioC.
On the same records, \jev{} separates gold evidence better than two LLMs and is 3--11 times faster.
\end{abstract}

%% file: sections/introduction.tex
\section{Introduction}
\label{sec:intro}

An assistant that talks with the same person for months accumulates more history than it can reread at every turn.
Rereading everything grows more expensive with every conversation and stops being possible once the history outgrows the context window.
Long-term memory systems therefore decide which parts of the history the assistant sees, and most of them decide it when the history is written: Mem0 extracts and updates facts~\citep{chhikara2025mem0}, Zep builds a temporal knowledge graph~\citep{rasmussen2025zep}, and MemOS, EverMemOS, Nemori and MIRIX organize conversations into memory units, episodes or typed stores~\citep{li2025memos,evermemos2026,nan2025nemori,wang2025mirix}.

\paragraph{The cost of write-time extraction.}
Rewriting a conversation at write time pays for understanding before the question is known.
Every message is processed whether or not it is ever asked about, and what the extractor drops or distorts cannot be recovered when a question finally reveals what mattered.
It also moves judgment away from the moment with the most information for it: the question, the recent dialogue and the candidate records are known together only when the question is asked.
And it ties memory to a schema: the extractor decides in advance what counts as a fact, an entity or a preference, so each new kind of data, such as documents, tasks or logs, needs a new extraction schema.

\paragraph{Fast judgments, slow thoughts.}
Dual-process accounts of thinking separate a fast, automatic System~1 from a slow, deliberate System~2~\citep{kahneman2011thinking}.
The work of memory divides the same way.
Most of it is System~1 work: small, independent yes/no judgments with explicit criteria, such as whether the reply should use this record, whether it is no longer current or whether it gives the second of the two dates the question needs.
Decision models such as \jev{} make such judgments by the dozen in a third of a second~\citep{jev2026}.
Only a little is System~2 work: writing a few search queries, naming what the reply needs and composing the answer from what is shown, which an LLM does well but slowly.
Because the judging is fast, \sys{} can read raw, dated records when a question arrives instead of rewriting them in advance: an LLM plans (System~2), \jev{} judges (System~1), and rules with explicit budgets decide which searches to page or rewrite and what to show.
Slow work that no reply can wait for runs in the background: for evidence that no question points to, such as an instruction given once, every mention of a topic or a value that changed, \sys{} consolidates each record once into an index of topic timelines, value histories and standing instructions that links back to the records.
The index directs reading, and the records remain the evidence.

\paragraph{Memory without a write-time schema.}
Deferring interpretation to read time also frees memory from the structure of its store.
Nothing about a record is decided when it is written, so \sys{} needs from a store only a search route that returns dated records: a conversation journal, a collection of documents or a task list can feed the same agent without a new extractor.
The consolidated index sits on top of the records and points back to them, so it never becomes the store's schema.
Because nothing stored has to be rewritten, the same records also serve longer histories, stronger answering models and new ways of reading them.
Memory of this kind can be added wherever records can be searched, which makes it a general way to give agents a past beyond conversation; we evaluate it on conversational memory, the setting with public benchmarks.

\paragraph{Our contributions.}
\sys{} runs as a \emph{replica}, a second instance of an open-source agent harness~\citep{dsh2026}, and hands the unchanged main agent one View per turn.
Our contributions are:
\begin{itemize}
  \item \textbf{Memory as System~1 and System~2} (\cref{sec:setting,sec:design,sec:impl}): a memory agent that gives judging to a decision model and planning and answering to an LLM, connected by rules under explicit budgets that use only the order and the yes/no of judgments.
  \item \textbf{Memory without a write-time schema} (\cref{sec:setting,sec:memory,sec:jevmem}): raw records about which nothing is decided until a question arrives, which no write-time processing can improve on in information (\cref{eq:dpi}), and a consolidated index that points back to them, so that the agent reads any store that returns dated records. Under one protocol, \sys{} is \jmDiff{} points more accurate than a concurrent system that uses the same decision model to organize memory at write time.
  \item \textbf{Accuracy at small context} (\cref{sec:main}): compared with the 14 memory systems that OmniMemEval re-evaluated with gpt-4.1-mini answering~\citep{omnimemeval2026}, \sys{} is the most accurate on LoCoMo (\finLoCoMo\%) and second on \lme{} (\finLME\%), from under 4k tokens of context per question, with the lowest effective cost index on LoCoMo.
  \item \textbf{A stronger System~2, bounded cost} (\cref{sec:reasoning,sec:scale}): with a reasoning model as System~2, \sys{} reaches \dsFinLoCoMo\% on LoCoMo and \dsFinLME\% on \lme, on par with the best published results on \lme; from BEAM-100K to BEAM-10M, with 80 times as many records, its cost per question grows by a factor of \beamTenCostRatioC.
  \item \textbf{Judging belongs to System~1} (\cref{sec:system1}): on the same records, \jev{} separates gold evidence better than two LLMs (AUC 0.942 against 0.900 and 0.853) and is 3--11 times faster.
\end{itemize}
\Cref{fig:tradeoff} summarizes the comparison on the two most widely reported benchmarks: \sys{} is the only system above 80\% on both that sends the answering model fewer than 4k tokens per question.

\begin{figure}[t]
\centering
\includegraphics[width=\linewidth]{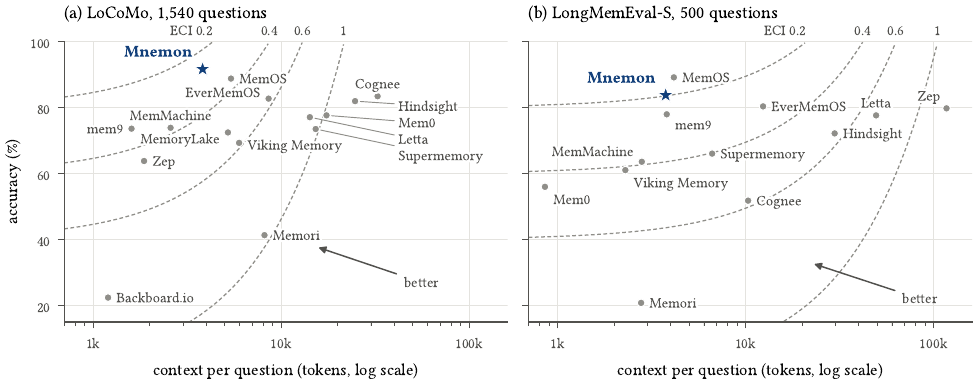}
\caption{Accuracy against context per question for \sys{} (star) and the 14 systems re-evaluated by OmniMemEval, all with gpt-4.1-mini answering. Dashed lines join points of equal effective cost index (\cref{eq:eci}); up and to the left is better.}
\label{fig:tradeoff}
\end{figure}

%% file: sections/related.tex
\section{Related Work}
\label{sec:related}

\paragraph{Memory for LLM agents.}
MemGPT, now Letta, pages information in and out of the context window~\citep{packer2023memgpt,letta2024}.
Most later systems work at write time: they extract facts~\citep{chhikara2025mem0}, build temporal knowledge graphs~\citep{rasmussen2025zep} or linked notes~\citep{xu2025amem}, organize memory into tiers, cubes or typed stores~\citep{kang2025memoryos,li2025memos,wang2025mirix}, consolidate conversations into episodes~\citep{evermemos2026,nan2025nemori}, or compress and restructure what they extract~\citep{fang2025lightmem,simplemem2026,premem2025,hmem2026,hymem2026,leanmem2026,lychee2026,emem2025,swiftmem2026,memori2026,hindsight2025}.
MemMachine keeps raw episodes indexed by sentence~\citep{memmachine2026}, and SmartSearch reranks raw history~\citep{smartsearch2026}.
\sys{} keeps raw records as the only evidence, reasons at read time, and consolidates only into an index that points back to the records.

\paragraph{Retrieval and judgment.}
\sys{} retrieves with BM25~\citep{robertson2009bm25} and dense embeddings~\citep{nussbaum2024nomic}, fused by reciprocal rank fusion (RRF)~\citep{cormack2009rrf}, and with hypothetical-document queries~\citep{gao2023hyde}, as in retrieval-augmented generation~\citep{lewis2020rag}.
\jev{} acts as a reranker~\citep{nogueira2019passage,sun2023rankgpt} whose judgments answer explicit propositions and drive actions, as in ReAct and Self-RAG~\citep{yao2023react,asai2024selfrag}.

\paragraph{Fast and slow thinking.}
Dual-process accounts separate fast, automatic judgment from slow deliberation~\citep{kahneman2011thinking}, and AI systems have borrowed the division to pair fast and slow components~\citep{booch2021thinking}, for instance a small, fast action model with an LLM planner in interactive agents~\citep{lin2023swiftsage}; LLM cascades route easy queries to cheap models and hard ones to expensive models~\citep{chen2023frugalgpt}.
Memory systems have recently adopted the division as well: D-Mem falls back from vector retrieval to exhaustive LLM reading~\citep{you2026dmem}, DCPM and Engram pair a fast write path with slow consolidation into schemas or a bi-temporal knowledge graph~\citep{fei2026dcpm,wang2026engram}, and Jev-Mem uses \jev{} to type and link each turn at write time and to steer retrieval over the resulting graph~\citep{jiang2026jevmem}.
\sys{} applies the division on the read path instead: System~1 judges raw records once a question is known, System~2 plans the searches and composes the answer, and nothing is decided when a record is written.

\paragraph{Evaluating memory.}
LoCoMo~\citep{maharana2024locomo} and \lme~\citep{wu2025longmemeval} are the standard benchmarks for long conversations; BEAM reaches 10M tokens~\citep{beam2026}, and HaluMem measures hallucination~\citep{halumem2025}.
Protocols and LLM graders~\citep{zheng2023judging} vary between papers, and several self-reported scores fell by 10--40 points when OmniMemEval re-evaluated 14 systems under one protocol~\citep{omnimemeval2026}.
We therefore follow that protocol and list published claims separately.

%% file: sections/setting.tex
\section{Problem Setting}
\label{sec:setting}

We consider an assistant, the \emph{main agent}, that converses with one user over many sessions.
Its \emph{history} is a sequence of sessions, each a dated sequence of messages.
At each user turn a \emph{memory agent} may read the history, and it publishes a \emph{View}: a bounded block of context that the main agent receives with the message before it replies.
The main agent is otherwise unchanged, and the memory agent never edits its reply.
In the benchmarks we use, a question is asked in a fresh session after its history has been written, so the answer depends on the View alone.

\paragraph{Formal model.}
We write the history as a sequence of dated records $H = (r_1, \ldots, r_n)$, each a timestamp and a short span of dialogue.
At a user turn, let $M$ be the recent dialogue, ending with the user's message, and let $Y$ be what the reply should convey; $Y$ depends on the message, which is not known when the history is written.
A memory agent maps $M$ and $H$ to a View $V = F(M, H)$ within a budget $B$ on its length, and the main agent replies from $M$ and $V$.
A \emph{write-time} memory first applies a transformation $\phi$ fixed before any question, such as extraction into facts or a graph, and reads only its output: $V = F\bigl(M, \phi(H)\bigr)$.
Since $\phi(H)$ is a function of $H$, the variables form a Markov chain $Y \to (M, H) \to \bigl(M, \phi(H)\bigr)$, and the data processing inequality~\citep{cover2006elements} gives
\begin{equation}
\label{eq:dpi}
I\bigl(Y;\, M, \phi(H)\bigr) \le I\bigl(Y;\, M, H\bigr),
\end{equation}
with equality only if $\phi(H)$ retains all that $H$ says about the answer to every question that may be asked.
Processing at write time can therefore lose information about questions not yet asked, and it cannot add any.
\Cref{eq:dpi} does not make raw records sufficient: the View must still fit $B$, and choosing it is the read-time work that \sys{} divides between two systems.
An index $X = \psi(H)$ kept beside the records, rather than in their place, loses nothing, since $I(Y; M, H, X) = I(Y; M, H)$; it can only make the relevant records easier to find.

\paragraph{Two kinds of work.}
Following dual-process accounts of thinking~\citep{kahneman2011thinking}, we divide the read-time work into two kinds.
\emph{System~1 work} consists of \emph{judgments}.
A judgment evaluates a yes/no proposition $\pi$ with an explicit criterion about one item $x$, a record or an index item, in a shared state $c$, such as the recent dialogue and the needs under consideration, and returns the probability $p_\pi(x \mid c) \in [0, 1]$ that $\pi$ holds: does the reply need this record, is it no longer current, does it satisfy this need.
Judgments on one state that take no other judgment's result as input form a \emph{wave}, which a decision model answers in parallel batches, dozens of judgments in a fraction of a second; the time a turn spends on System~1 work therefore grows with its number of waves, not of judgments.
\emph{System~2 work} is open-ended generation $g(c)$ that does not decompose into such propositions: writing queries for records not yet seen, naming what a reply needs, and composing an answer that counts, compares dates or chains facts.
It suits an LLM, which does it well but slowly, so a memory agent should keep as little of it on the read path as it can and move the rest to the background.
We borrow the division of labor, not the view that fast judgment is error-prone: our judgments have explicit criteria, and a model built for them makes them well (\cref{sec:system1}).

\paragraph{Harness.}
We build the memory agent in DeepSeek Harness (DSH), an open-source agent harness in which every capability is a plugin~\citep{dsh2026}.
Its dsh-mnemon suite provides memory through \emph{Sources}, which expose stores such as a journal of records through read and search routes, and \emph{Strategies}, which decide what the harness publishes into the agent's context as one View per turn~\citep{mnemon2026code}.
\sys{} is agnostic to the structure of a store: it needs from a Source only a search route that returns dated records, so a store of documents, tasks or past sessions can feed the same View as the conversation journal.
In this paper all memory lives in one Source, the journal.

\paragraph{Cost measure.}
A memory system incurs cost at three points: when it processes the history (write time), when it decides what to show (read time) and when the answering model reads what it was shown (answer time).
Only the last is reported for every system in public re-evaluations, as the number of tokens of context sent to the answering model per question~\citep{omnimemeval2026}.
Accuracy per unit of cost rewards systems that are cheap and often wrong, so we price errors instead.
If a wrong answer is repaired by one full-context answer of cost $c_\text{full}$, a system with accuracy $a$ that sends $c$ tokens of context per question has the expected cost
\begin{equation}
\label{eq:eci}
\mathrm{ECI} = (1-a) + \frac{c}{c_\text{full}}
\end{equation}
in units of $c_\text{full}$, which we call its \emph{effective cost index}.
Lower is better, and answering from the full history costs at least one.
Because read-time and write-time costs are not published for other systems, we compare costs across systems by this index alone and report our other costs separately.

%% file: sections/system.tex
\section{Design}
\label{sec:design}

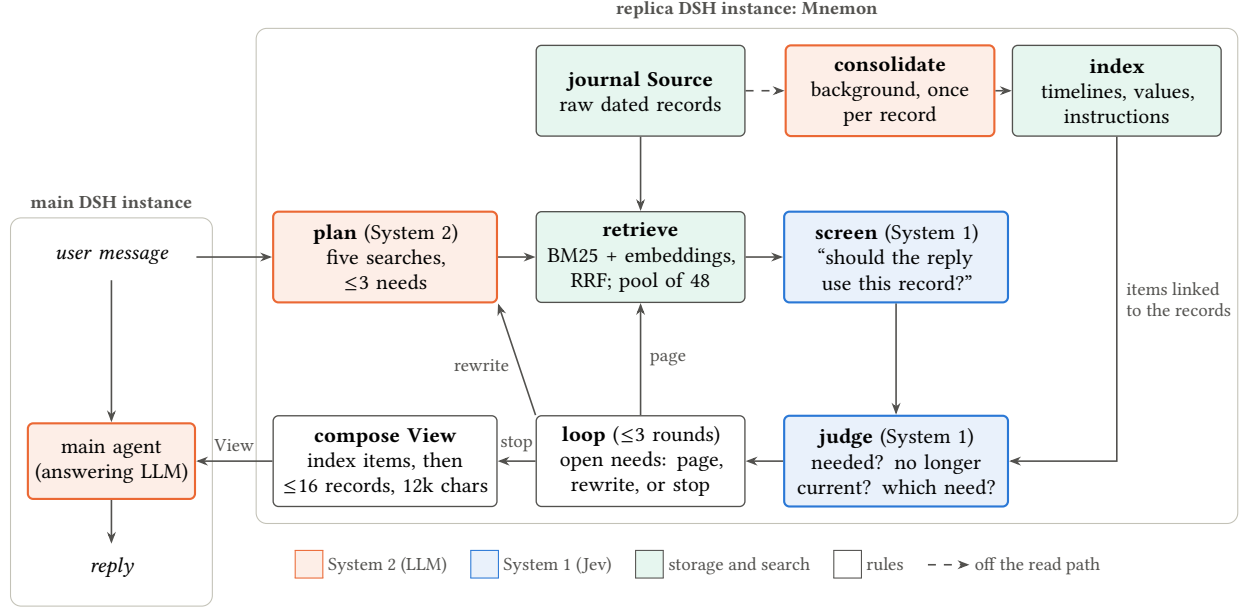
\begin{figure}[t]
\centering
\begin{tikzpicture}[
  font=\footnotesize,
  box/.style={draw=ink2, rounded corners=2pt, line width=0.5pt, align=center, inner sep=3pt, text width=2.55cm, minimum height=1.2cm, fill=white},
  wide/.style={text width=2.75cm},
  llm/.style={box, fill=fillB, draw=other, line width=0.8pt},
  sone/.style={box, fill=fillA, draw=ours, line width=0.8pt},
  store/.style={box, fill=fillC},
  io/.style={align=center, text width=2.0cm, minimum height=0.6cm, font=\footnotesize\itshape},
  flow/.style={-{Stealth[length=4.5pt]}, line width=0.6pt, draw=ink2},
  lab/.style={font=\scriptsize, text=ink2},
]
\node[io] (user) at (0,0) {user message};
\node[llm, text width=2.0cm, minimum height=1.0cm] (main) at (0,-2.7) {main agent\\(answering LLM)};
\node[io] (reply) at (0,-4.1) {reply};
\draw[flow] (user) -- (main);
\draw[flow] (main) -- (reply);
\node[llm, wide] (plan) at (3.62,0) {\textbf{plan} (System~2)\\five searches,\\$\le$3 needs};
\node[store] (retrieve) at (7.0,0) {\textbf{retrieve}\\BM25 + embeddings,\\RRF; pool of 48};
\node[sone, wide] (screen) at (10.38,0) {\textbf{screen} (System~1)\\``should the reply\\use this record?''};
\node[sone, wide] (judge) at (10.38,-2.7) {\textbf{judge} (System~1)\\needed? no longer\\current? which need?};
\node[box] (loop) at (7.0,-2.7) {\textbf{loop} ($\le$3 rounds)\\open needs: page,\\rewrite, or stop};
\node[box, wide] (compose) at (3.62,-2.7) {\textbf{compose View}\\index items, then\\$\le$16 records, 12k chars};
\node[store] (journal) at (7.0,2.2) {\textbf{journal Source}\\raw dated records};
\node[llm] (consolidate) at (10.3,2.2) {\textbf{consolidate}\\background, once\\per record};
\node[store] (index) at (13.3,2.2) {\textbf{index}\\timelines, values,\\instructions};
\draw[flow] (user) -- (plan);
\draw[flow] (plan) -- (retrieve);
\draw[flow] (retrieve) -- (screen);
\draw[flow] (screen) -- (judge);
\draw[flow] (judge) -- (loop);
\draw[flow] (loop.north) -- node[lab, right] {page} (retrieve.south);
\draw[flow] (loop.north west) -- node[lab, left, pos=0.45] {rewrite} (plan.south east);
\draw[flow] (loop) -- node[lab, above] {stop} (compose);
\draw[flow] (journal) -- (retrieve);
\draw[flow, dashed] (journal) -- (consolidate);
\draw[flow, dashed] (consolidate) -- (index);
\draw[flow] (index.south) |- node[lab, pos=0.25, right, align=left] {items linked\\to the records} (judge.east);
\draw[flow] (compose) -- node[lab, above] {View} (main);
\begin{scope}[on background layer]
\node[draw=rule, rounded corners=4pt, fit=(user)(main)(reply), inner sep=6pt, label={[lab, font=\scriptsize\bfseries]above:{main DSH instance}}] {};
\node[draw=rule, rounded corners=4pt, fit=(plan)(retrieve)(screen)(judge)(loop)(compose)(journal)(consolidate)(index), inner sep=6pt, label={[lab, font=\scriptsize\bfseries]above:{replica DSH instance: \sys}}] {};
\end{scope}
\node[lab, anchor=north west, align=left] at (2.3,-3.75) {\fcolorbox{other}{fillB}{\rule{0pt}{4pt}\rule{4pt}{0pt}} System~2 (LLM) \quad \fcolorbox{ours}{fillA}{\rule{0pt}{4pt}\rule{4pt}{0pt}} System~1 (\jev) \quad \fcolorbox{ink2}{fillC}{\rule{0pt}{4pt}\rule{4pt}{0pt}} storage and search \quad \fcolorbox{ink2}{white}{\rule{0pt}{4pt}\rule{4pt}{0pt}} rules \quad \tikz[baseline=-0.6ex]\draw[flow, dashed] (0,0) -- (0.55,0); off the read path};
\end{tikzpicture}
\caption{\sys{} at one user turn. System~2, an LLM, plans the searches; System~1, \jev, screens and judges the raw records they return and the index items those records link to; rules loop on needs that are still open and publish a View into the main instance, whose LLM answers. Off the read path, System~2 consolidates each record once into an index that points back to the records.}
\label{fig:architecture}
\end{figure}

At every user turn \sys{} decides what the main agent should see before it replies (\cref{fig:architecture}).
It divides the work as \cref{sec:setting} does: an LLM plans the searches and names what the reply needs (System~2), \jev{} judges the pooled records (System~1), rules with explicit budgets connect the two, and slower System~2 work, consolidating the history into an index, runs in the background.
\Cref{sec:memory,sec:read,sec:rules} describe the memory, the two systems that read it and the rules; \cref{alg:turn} summarizes one turn.

\subsection{Memory: Raw Records and a Consolidated Index}
\label{sec:memory}

\paragraph{Records.}
Each session is written as it happened, as short \texttt{speaker: text} chunks, each headed by the session's number and date.
No LLM is on the write path of a record, and nothing about it is decided when it is written; the only computation is one embedding, so the write path is the same whatever the records are about.

\paragraph{Consolidation.}
Questions about a whole conversation, such as every mention of a topic, the current value of something that changed or an instruction given once, need evidence that the question's own searches rarely find.
Gathering it is System~2 work too slow for the read path, so a background job reads the journal once, in order, and a small LLM folds it, batch by batch, into an \emph{index} of three kinds of items, each linked to the records it came from: events on topic timelines; dated histories of values that can change, marked where a value contradicts an earlier one; and standing instructions and preferences.
Items point to records and never replace them: the index is an overlay on the records, not their schema, and \sys{} reads the records with or without it.

\paragraph{Why raw records.}
Keeping records as written separates the work that must be fast and complete from the work that can be slow and fallible.
A record can be searched as soon as it is embedded, so a history of about 100k tokens is written in under a second (\cref{sec:impl}); organizing it runs in the background, can lag or fail without losing evidence, and can be redone from the records.
Memory kept this way grows in three directions without rewriting what is stored: to longer histories, since only search grows with them (\cref{eq:work}); to stronger models, since a better System~2 reads the same records without re-ingesting the history (\cref{sec:reasoning}); and to changes in reading, since new rules and budgets, or a decision model that orders and decides alike (\cref{prop:calibration}), take effect on the whole history at once, and a new kind of index item needs only another background pass.

\subsection{Reading with Two Systems}
\label{sec:read}

\paragraph{Plan (System~2).}
Planning is System~2 work: it must anticipate how the records that answer a question are worded and what the reply will have to say.
Two LLM calls read the last six messages in parallel.
The \emph{search plan} writes three searches in the words a stored record would use, plus one hypothetical record that would answer the message (HyDE~\citep{gao2023hyde}); with the message itself, a turn starts with five searches.
The \emph{need plan} lists at most three information needs, marked \texttt{NEED}, or \texttt{ALL} when the reply needs every instance of something (a count, a total, a list).
Neither call reads the journal, so their cost does not grow with the history.

\paragraph{Retrieve.}
Each search ranks the journal's records by BM25 and by embedding similarity, fuses the two rankings by RRF~\citep{cormack2009rrf}, and returns a page of records.
The pages merge into a \emph{pool}: each search first contributes an equal quota of its top results, and the remaining slots go to the highest summed reciprocal rank.

\paragraph{Judge (System~1).}
Judging is System~1 work: each proposition concerns one record, has an explicit criterion and can be answered independently of the others.
\jev{} answers typed yes/no propositions and returns the probability of yes~\citep{jev2026}.
It \emph{screens} each pooled record (should the reply use this record?), then \emph{judges} the records that could enter the View, jointly and with the last six messages, as \emph{needed} or \emph{no longer current} (a later message or record explicitly changes, cancels or completes it), and fills a \emph{need table}: does this record satisfy this need?
\jev{} then judges which index items the reply needs, among those linked to the judged records and the dozen nearest to the message.
Counting, arithmetic, comparing dates and multi-hop reasoning are System~2 work, which \jev{} handles poorly~\citep{jevdocs2026}; they are left to the planner and the answering model.

\subsection{Rules and Budgets}
\label{sec:rules}

The rules connect the two systems.
They read only what System~1 reliably gives, its yes/no decision ($p \ge 1/2$) and its ranking, together with events observed in the previous round, and they call on System~2 only when System~1 reports trouble: a need that no judged record satisfies prompts the planner for a new search, much as System~2 is mobilized when System~1 runs into difficulty~\citep{kahneman2011thinking}.
Every other constant is an interpretable budget: three needs per turn, pages of 20 records or 12,000 characters, a pool of 48 records, three rounds, and a View of 16 records and 12,000 characters.
Because the rules use judgments only through their order and their decisions, they do not depend on how the decision model is calibrated.

\begin{proposition}[Calibration invariance]
\label{prop:calibration}
Let $h$ be a strictly increasing map of $[0, 1]$ onto itself with $h(1/2) = 1/2$.
Replacing every judgment $p$ by $h(p)$ leaves every action of the loop, and hence the View, unchanged.
\end{proposition}
\begin{proof}
The rules compare judgments only with each other or with $1/2$, and $h$ preserves both comparisons.
\end{proof}

\noindent A decision model can thus be replaced by another that orders and decides alike, without retuning any rule.

\paragraph{The loop.}
Each open need gets one action per round.
A single need is met once a record satisfies it.
A need for every instance, labeled \texttt{ALL} or observed because two records satisfy it, pages each search that found a new positive until none does.
A need that no record satisfies yet follows its \emph{lead}, the record the need table ranks highest: the loop pages the search that ranked the lead highest, then asks the planner for one new search, then drops the need.

\paragraph{The View.}
The View leads with the index, each block within a fixed budget: standing instructions, a one-line directory of topics, and the timelines and value histories \jev{} judged needed.
Records follow: those judged needed, then, when the planner named a need, the rest in \jev's ranking up to the View's budget, trimmed from the end to make room for the index.
Records judged no longer current are left out.

\paragraph{Work per turn.}
With these budgets, the work on a turn's critical path does not grow with the history, except in search.
A turn that makes $R$ sequential System~2 calls (the plan, any new search the loop asks for, and the answer), $W$ waves of judgments and $S$ reads of the journal takes
\begin{equation}
\label{eq:work}
T = \sum_{i=1}^{R} \tau^{(2)}_i + \sum_{w=1}^{W} \tau^{(1)}_w + \sum_{s=1}^{S} \tau^{\mathrm{read}}_s(n),
\end{equation}
where $\tau^{(2)}_i$, $\tau^{(1)}_w$ and $\tau^{\mathrm{read}}_s(n)$ are the latencies of the $i$th System~2 call, the $w$th wave and the $s$th read, each read being over the $n$ records of the journal.
The rounds and the needs bound $R$, $W$ and $S$, and only reads depend on $n$; since reads spend no model tokens, the tokens a turn spends are bounded as well.

\begin{algorithm}[t]
\caption{One user turn of \sys.}
\label{alg:turn}
\small
\begin{algorithmic}[1]
\Require the last six messages $M$, the journal $J$ and its index $X$
\State $(Q, N) \gets \textsc{Plan}(M)$ \Comment{System~2: five searches, at most three needs}
\State $P \gets \textsc{Pool}(\{\textsc{Search}(J, q) : q \in Q\})$ \Comment{BM25 and embeddings, fused by RRF; at most 48 records}
\For{round $= 1, 2, 3$}
  \State $\textsc{Screen}(P)$;\; $\textsc{Judge}(P, M, N)$ \Comment{System~1: use it? needed? no longer current? which need?}
  \State $A \gets \textsc{Act}(N, P)$ \Comment{rules: page a search, ask System~2 for a new one, or drop the need}
  \If{$A = \emptyset$} \State \textbf{break} \EndIf
  \State $P \gets P \cup \textsc{Run}(A)$
\EndFor
\State $I \gets \textsc{JudgeItems}(X, P, M)$ \Comment{index items linked to judged records, and the nearest}
\State \Return $\textsc{Compose}(I, P)$ \Comment{index first, then records, within the View's budgets}
\end{algorithmic}
\end{algorithm}

%% file: sections/implementation.tex
\section{Implementation}
\label{sec:impl}

\sys{} runs as a \emph{replica}: a second DSH instance with the same Sources, its own Strategy and a control loop that calls \jev.
At each user turn the replica runs one job and publishes its View into the main instance.
Upstream DSH is unmodified, and apart from 115 changed lines of dsh-mnemon~\citep{mnemon2026code} the replica consists of new plugins, released with this paper~\citep{mnemon2026artifact}.

\paragraph{Journal.}
Records are chunks of up to about 600 characters, embedded locally with nomic-embed-text~\citep{nussbaum2024nomic}; a \lme{} history of about 106k tokens becomes roughly 1,000 records, written in under a second.
For BEAM-10M, the journal is sized for up to 200,000 records per conversation and keeps its search index between searches.

\paragraph{Consolidation.}
Background jobs read the journal from a watermark, in batches of about 10,000 characters that hold user messages whole and the first 240 characters of each assistant message, about 15\% of a conversation's text.
DeepSeek-V4.1-Flash without thinking~\citep{deepseekflash2026} folds each batch into the index.
For very long histories a fold sees at most 120 topics, 240 values and 60 instructions, half chosen by the words they share with the batch and half by recency.
Index items are embedded as they are written, so that no View waits for them.

\paragraph{Models.}
System~2 is an LLM: the planner and the answering model are the same LLM, gpt-4.1-mini or DeepSeek-V4.1-Flash (\cref{sec:setup}).
System~1 is \jev{} 1.13, which judges up to 40 records per call and answers every proposition about them in that call (\cref{sec:system1}).

%% file: sections/evaluation.tex
\section{Evaluation}
\label{sec:eval}

We evaluate \sys{} on four public benchmarks and ask six questions.
How accurate is it at small context, compared with systems re-evaluated under one protocol (\cref{sec:main})?
What does a stronger System~2 add, and how does the result compare with the best published results (\cref{sec:reasoning})?
Does its cost stay bounded as the history grows to 10M tokens (\cref{sec:scale})?
Is judging better done by System~1 than by an LLM (\cref{sec:system1})?
Is the same System~1 better used to judge raw records at read time than to organize memory at write time (\cref{sec:jevmem})?
And what does consolidation add (\cref{sec:consolidation})?

%% file: sections/protocol.tex
\subsection{Setup}
\label{sec:setup}

\paragraph{Benchmarks.}
LoCoMo~\citep{maharana2024locomo} has 10 conversations of 13--25k tokens and 1,540 non-adversarial questions.
\lme~\citep{wu2025longmemeval,longmemeval_cleaned} has 500 questions, each with its own history of about 106k tokens.
HaluMem~\citep{halumem2025} (Medium set) has 20 users and 3,467 questions.
BEAM~\citep{beam2026} has two tiers: BEAM-100K, with 20 conversations and 400 questions over 10 abilities, and BEAM-10M, with 10 conversations of about 90,000 records each and 200 questions.
LoCoMo is reported with its original labels; for comparison with published claims we also give our score under our revised labels, which drop 44 unusable questions and correct 25 gold answers.

\paragraph{Models.}
In the \emph{standard setting}, gpt-4.1-mini~\citep{openai2025gpt41} (snapshot 2025-04-14, temperature 0) answers and plans, as in OmniMemEval~\citep{omnimemeval2026}.
In the \emph{reasoning setting}, DeepSeek-V4.1-Flash~\citep{deepseekflash2026} answers with thinking and plans without.

\paragraph{Grading.}
LoCoMo answers are graded with the lenient prompt of Mem0~\citep{chhikara2025mem0} and later work, \lme{} answers with its per-type prompts, and HaluMem and BEAM answers with their official prompts and rubrics.
The primary grader is gpt-4.1-mini in the standard setting and DeepSeek-V4.1-Flash (without thinking, temperature 0) in the reasoning setting, and the other grades every answer as well; on LoCoMo and \lme{} the two agree on 96--97\% of answers (Cohen's $\kappa$ 0.85--0.89).
OmniMemEval grades with gpt-4o-mini, so comparisons with it carry a grader difference of 1--2 points.
Paired differences between configurations are reported with 95\% confidence intervals (CIs).

\paragraph{Baselines.}
We compare \sys{} with the 14 systems that OmniMemEval re-evaluated with gpt-4.1-mini answering, using its published numbers, and, separately, with the best result each open-source project has published, whatever its model and protocol.

\paragraph{Cost accounting.}
Costs follow the token counts the APIs return, at list prices per million tokens (gpt-4.1-mini \$0.40 input, \$0.10 cached, \$1.60 output; DeepSeek-V4.1-Flash off-peak \$0.15, \$0.003, \$0.60; \jev{} \$0.042 input), and cover the answer, the planner and \jev; consolidation is a one-time cost per history.
Questions that failed on a transient error or received an empty View were answered once more.

%% file: sections/results.tex
\subsection{Accuracy at Small Context}
\label{sec:main}

\Cref{tab:omnimemeval} places \sys{} among the 14 systems that OmniMemEval re-evaluated with the same answering model.
On LoCoMo it scores \finLoCoMo\%, against 88.83\% for MemOS and at most 83.48\% for the rest, a margin beyond the 1--2 points that separate graders; on \lme{} it scores \finLME\%, second to MemOS (89.2\%) and ahead of EverMemOS (80.4\%) and Zep (79.8\%).
It sends the answering model about \finCtxLoCoMo{} tokens per question, less than any other system above 80\% on either benchmark (MemOS 4.2--5.4k, EverMemOS 8.6--12.4k, Hindsight 24.7k, Cognee 32.5k).
The context is small because System~1 does the broad reading: per question, \jev{} reads \jevTokRange{} tokens of records, \jevRatioRange{} times what the answering model reads, at about a tenth of its price per token, and passes on only what the reply needs.
With $c_\text{full}$ set to 21,613 gpt-4.1-mini input tokens on LoCoMo and 105,134 on \lme, \sys{} has the lowest effective cost index (\cref{eq:eci}) of the 15 systems on LoCoMo, \finEciLoCoMo{} against \eciOtherLoCoMo{} for \eciOtherNameLoCoMo, and the second lowest on \lme, \finEciLME{} after \eciOtherNameLME{} at \eciOtherLME{} (\cref{fig:tradeoff}).

\input{tables/omnimemeval}

By question type (\cref{tab:by-type}), \sys{} leads on every LoCoMo type except open-domain, whose questions ask for inferences the conversation does not state, with the largest margins on temporal and multi-hop questions.
On \lme{} it leads on knowledge updates, matches MemOS on multi-session questions and trails it mainly on temporal-reasoning and preference questions, which require computing or inferring at answer time.

\input{tables/by_type}

\subsection{A Stronger System~2}
\label{sec:reasoning}

System~2 can be replaced without changing the memory.
With DeepSeek-V4.1-Flash planning and answering, with thinking when it answers, \sys{} scores \dsFinLoCoMo\% on LoCoMo and \dsFinLME\% on \lme{} under the DeepSeek grader (\dsFinLoCoMoMini\% and \dsFinLMEMini\% under gpt-4.1-mini), at \$\dsCostLoCoMo{} and \$\dsCostLME{} per 1,000 questions.
On \lme{} the stronger System~2 adds \readerGainLME{} points on the same memory under the gpt-4.1-mini grader (95\% CI \readerGainLMECI), most on temporal-reasoning and multi-session questions, whose computing and counting are System~2 work; with it, \sys{} reaches or passes MemOS on five of the six \lme{} question types (\cref{tab:by-type}). Raw records leave computing and counting to the answering model, and a stronger one does them.
\Cref{tab:best-reported} lists each open-source project's best published result, whatever its answering model and protocol.
On \lme{} \sys{} is on par with the best claims (Hindsight 94.6\%, Mem0 94.4\% with gpt-5); on LoCoMo, where claims reach 94.7\% (Zep with gpt-5.4), it scores \dsFinLoCoMo\% on the original labels and \dsFinLoCoMoRev\% on the revised ones.

\input{tables/best_reported}

\subsection{Other Benchmarks and Scale}
\label{sec:scale}

\Cref{tab:ours} reports \sys{} on every benchmark and tier with its full cost per question.
HaluMem and BEAM test what the design targeted least: HaluMem's grader counts an answer that adds details beyond the reference as a hallucination, and BEAM asks for instructions to be followed, contradictions to be reported and whole conversations to be summarized.
On them, \sys{} ranks \finHaluRankWord{} of \haluSystems{} systems on HaluMem and \finBeamRankWord{} of \beamSystems{} on BEAM-100K.

\input{tables/ours}

Nothing on the read path of \sys{} grows with the history except the search index: the planner reads the recent dialogue, \jev{} screens at most 48 records a round, and the View has fixed budgets.
BEAM-10M tests this with conversations of about 90,000 records, 80 times as many as BEAM-100K's.
From BEAM-100K to BEAM-10M the cost per question grows by a factor of \beamTenCostRatioC{}, and the work on a question's critical path stays about the same.
On BEAM-10M \sys{} scores \beamTenC, against 43.3--59.8 for the 11 systems OmniMemEval reports, and consolidating its 10 conversations took \beamTenBatches{} folds without a failure.

We report latency as critical-path work (\cref{eq:work}, \cref{tab:ours}).
For the median question, System~2 makes two sequential calls, the plan (two requests in parallel) and the answer, and a third in the \workReplan\% of questions whose loop asks for a new search; System~1 makes \workJevCalls{} calls in \workJevWaves{} waves, \workJevTime\,s at \workWave\,s a wave (\cref{fig:system1}b).
A warm search takes \workSearchSmall\,ms, mostly to embed the query, and a question's reads \workReadSmall\,s, but \workSearchTen\,ms and \workReadTen\,s on BEAM-10M's largest history (\workTenRecords{} records), because our search scores every record; inverted and approximate nearest-neighbor indexes would avoid this.

%% file: tables/omnimemeval.tex
\begin{gentable}[!htb]
\centering
\small
\caption{\sys{} and the 14 systems re-evaluated by OmniMemEval~\citep{omnimemeval2026}, all with gpt-4.1-mini answering: accuracy, context sent to the answering model per question, and effective cost index (ECI, \cref{eq:eci}; lower is better). OmniMemEval grades with gpt-4o-mini and we with gpt-4.1-mini; with DeepSeek as grader, \sys{} scores 91.4\% and 85.4\%. Best per column in bold.}
\label{tab:omnimemeval}
\setlength{\tabcolsep}{5pt}
\begin{tabular}{@{}l S[table-format=2.2] S[table-format=3.1] S[table-format=1.3] S[table-format=2.2] S[table-format=3.1] S[table-format=1.3]@{}}
\toprule
& \multicolumn{3}{c}{LoCoMo (1,540 questions)} & \multicolumn{3}{c}{\lme{} (500 questions)} \\
\cmidrule(lr){2-4}\cmidrule(l){5-7}
& {accuracy} & {context} & {ECI} & {accuracy} & {context} & {ECI} \\
system & {(\%)} & {(k tokens)} & & {(\%)} & {(k tokens)} & \\
\midrule
\sys{} & \bfseries 91.7 & 3.8 & \bfseries 0.259 & 83.8 & 3.8 & 0.198 \\
\midrule
MemOS & 88.83 & 5.4 & 0.362 & \bfseries 89.2 & 4.2 & \bfseries 0.147 \\
Cognee & 83.48 & 32.5 & 1.670 & 51.8 & 10.3 & 0.580 \\
EverMemOS & 82.75 & 8.6 & 0.569 & 80.4 & 12.4 & 0.314 \\
Hindsight & 81.99 & 24.7 & 1.322 & 72.2 & 29.8 & 0.561 \\
Mem0 & 77.68 & 17.4 & 1.028 & 56.0 & 0.9 & 0.448 \\
Letta & 77.12 & 14.2 & 0.885 & 77.67 & 49.4 & 0.693 \\
MemMachine & 73.9 & 2.6 & 0.380 & 63.6 & 2.8 & 0.391 \\
mem9 & 73.64 & 1.6 & 0.337 & 78.0 & 3.8 & 0.256 \\
Supermemory & 73.53 & 15.2 & 0.970 & 66.07 & 6.6 & 0.402 \\
MemoryLake & 72.49 & 5.2 & 0.516 & {--} & {--} & {--} \\
Viking Memory & 69.33 & 6.0 & 0.583 & 61.07 & 2.3 & 0.411 \\
Zep & 63.83 & 1.9 & 0.448 & 79.8 & 117.1 & 1.316 \\
Memori & 41.34 & 8.1 & 0.963 & 20.8 & 2.8 & 0.818 \\
Backboard.io & 22.4 & 1.2 & 0.831 & {--} & {--} & {--} \\
\bottomrule
\end{tabular}
\end{gentable}

%% file: tables/by_type.tex
\begin{gentable}[!htb]
\centering
\small
\caption{Accuracy (\%) by question type. \sys{} is graded by gpt-4.1-mini in the standard setting (gpt-4.1-mini as System~2) and the reasoning setting (DeepSeek-V4.1-Flash); MemOS and EverMemOS, the strongest systems OmniMemEval re-evaluated, as it reports them (gpt-4.1-mini answering, gpt-4o-mini grading). Bold: best of the three systems with gpt-4.1-mini answering.}
\label{tab:by-type}
\setlength{\tabcolsep}{5pt}
\begin{tabular}{@{}l r S[table-format=2.1] S[table-format=3.2] S[table-format=2.2] S[table-format=3.1]@{}}
\toprule
& & \multicolumn{3}{c}{gpt-4.1-mini answering} & {reasoning} \\
\cmidrule(lr){3-5}\cmidrule(l){6-6}
question type & questions & {\sys} & {MemOS} & {EverMemOS} & {\sys} \\
\midrule
\multicolumn{6}{@{}l}{\textit{LoCoMo}} \\
\quad single-hop & 841 & \bfseries 94.6 & 92.51 & 86.8 & 96.1 \\
\quad multi-hop & 282 & \bfseries 91.8 & 88.65 & 77.78 & 92.2 \\
\quad temporal & 321 & \bfseries 91.3 & 85.05 & 84.11 & 92.2 \\
\quad open-domain & 96 & 66.7 & \bfseries 69.79 & 57.29 & 67.7 \\
\addlinespace[2pt]
\multicolumn{6}{@{}l}{\textit{\lme}} \\
\quad single-session, user & 70 & 95.7 & \bfseries 100.0 & 91.43 & 97.1 \\
\quad single-session, assistant & 56 & 94.6 & \bfseries 100.0 & 89.29 & 100.0 \\
\quad single-session, preference & 30 & 76.7 & \bfseries 100.0 & 96.67 & 100.0 \\
\quad temporal reasoning & 133 & 76.7 & \bfseries 89.47 & 81.95 & 92.5 \\
\quad multi-session & 133 & 78.2 & \bfseries 78.95 & 66.17 & 88.0 \\
\quad knowledge update & 78 & \bfseries 89.7 & 84.62 & 79.49 & 93.6 \\
\bottomrule
\end{tabular}
\end{gentable}

%% file: tables/best_reported.tex
\begin{gentable}[!htb]
\centering
\scriptsize
\caption{Each open-source project's best published result, with whatever answering model, grader and protocol it used (checked September 26--29, 2026). The settings differ widely, so the table ranks claims, not systems. $^\dagger$On the revised LoCoMo labels (\cref{sec:setup}); 92.2 on the original labels, which every other entry uses. $^\ddagger$Including the adversarial category.}
\label{tab:best-reported}
\setlength{\tabcolsep}{3pt}
\begin{tabularx}{\linewidth}{@{}l S[table-format=2.2] S[table-format=2.1] >{\raggedright\arraybackslash}p{4.7cm}>{\raggedright\arraybackslash}X@{}}
\toprule
project & {LoCoMo} & {\lme} & answering model / grader & method \\
\midrule
\textbf{\sys} & 95.3{$^\dagger$} & 94.4 & DeepSeek-V4.1-Flash (thinking) / DeepSeek & raw records and a consolidated index, LLM-planned search, \jev{} judging \\
Zep / Graphiti~\citep{zeprepo} & 94.7 & 90.2 & gpt-5.4 (medium reasoning) / gpt-5.4 & temporal knowledge graph + reranking \\
EverMemOS~\citep{evermemos2026} & 93.05 & 83.0 & gpt-4.1-mini / three graders averaged & MemCell to MemScene \\
Mem0~\citep{mem0repo} & 92.5 & 94.4 & gpt-5 / gpt-5 & LLM-extracted facts, top 200 (about 7k tokens) \\
memU~\citep{memu2026} & 92.09 & {--} & not stated (early version) & -- \\
Hindsight~\citep{hindsightrepo,hindsight2025} & 92.0 & 94.6 & not stated (paper: gemini-3-pro 89.6 / 91.4) & four memory networks, 36--44k tokens of context \\
MemMachine~\citep{memmachine2026} & 91.69 & 93.0 & gpt-4.1-mini; \lme{} gpt-5-mini / gpt-4o-mini & raw episodes indexed by sentence \\
MemOS~\citep{omnimemeval2026} & 88.83 & 89.2 & gpt-4.1-mini / gpt-4o-mini (OmniMemEval) & MemCube tree / graph memory \\
Memori~\citep{memori2026} & 87.0 & {--} & gpt-4.1-mini / gpt-4.1-mini & LLM-extracted triples + summaries \\
mem9~\citep{mem9_2026} & 86.85 & {--} & qwen3.6-plus / qwen3.6-plus & keyword + vector retrieval \\
MIRIX~\citep{wang2025mirix} & 85.38 & {--} & gpt-4.1-mini / gpt-4.1 & multi-agent, six memory types \\
Nemori~\citep{nan2025nemori} & 83.05 & 74.6 & gpt-4.1-mini / not stated & episodic narratives + distilled facts \\
OpenViking~\citep{openviking2026} & 82.86 & {--} & Doubao 2.0 Pro / Doubao & file-system context store, three tiers \\
Jev-Mem~\citep{jiang2026jevmem} & 77.7{$^\ddagger$} & {--} & gpt-4o-mini / not stated & \jev{}-typed turns in a relation graph built at write time \\
Memobase~\citep{memobase2025} & 75.78 & {--} & not stated & user profile + event timeline \\
Letta~\citep{letta2024} & 74.0 & {--} & gpt-4o-mini / gpt-4.1 & agent searches raw conversation files \\
LightMem~\citep{fang2025lightmem} & 72.99 & 73.2 & gpt-4o-mini; \lme{} glm-4.6 / not stated & pre-compression + offline update \\
Supermemory~\citep{supermemory2026} & {--} & 85.2 & gemini-3-pro / not stated & -- \\
SimpleMem~\citep{simplemem2026} & {--} & 84.4 & gpt-4.1 / gpt-4.1-mini & atomic memory units \\
Engram~\citep{wang2026engram} & {--} & 83.6 & doubao-seed-2.0-pro / deepseek-v3.2 & bi-temporal knowledge graph from asynchronous extraction \\
\bottomrule
\end{tabularx}
\end{gentable}

%% file: tables/ours.tex
\begin{gentable}[!htb]
\centering
\small
\caption{\sys{} on each benchmark and tier with gpt-4.1-mini answering: score under each grader (accuracy; HaluMem: share correct; BEAM: rubric score), rank among the systems OmniMemEval re-evaluated, context and cost per question, one-time consolidation cost per history, and the median question's \jev{} calls (sequential waves) and searches, with the warm latency of one search on the largest history.}
\label{tab:ours}
\setlength{\tabcolsep}{4pt}
\begin{tabular}{@{}l S[table-format=2.1] S[table-format=2.1] c S[table-format=1.1] S[table-format=1.2] S[table-format=1.3] c S[table-format=2.0] S[table-format=3.0]@{}}
\toprule
& \multicolumn{2}{c}{score, grader} & & {context} & \multicolumn{2}{c}{cost (\$)} & \multicolumn{3}{c}{median question} \\
\cmidrule(lr){2-3}\cmidrule(lr){6-7}\cmidrule(l){8-10}
benchmark & {gpt-4.1-mini} & {DeepSeek} & rank & {(k tokens)} & {/1k q.} & {/history} & \jev{} calls (waves) & {searches} & {search (ms)} \\
\midrule
LoCoMo & 91.7 & 91.4 & 1/15 & 3.8 & 3.27 & 0.013 & 5 (4) & 7 & 14 \\
\lme & 83.8 & 85.4 & 2/13 & 3.8 & 3.33 & 0.017 & 5 (5) & 7 & 14 \\
HaluMem & 73.3 & 65.8 & 8/13 & 3.4 & 3.60 & 0.089 & 7 (6) & 10 & 18 \\
BEAM-100K & 64.5 & 60.5 & 10/12 & 3.8 & 4.80 & 0.012 & 9 (7) & 13 & 17 \\
BEAM-10M & 51.2 & 48.8 & 10/12 & 3.8 & 5.32 & 1.62 & 10 (7) & 14 & 248 \\
\bottomrule
\end{tabular}
\end{gentable}

%% file: sections/analysis.tex
\subsection{System~1 Against LLMs}
\label{sec:system1}

The division gives judging to System~1.
To test it, we gave gpt-4.1-mini and DeepSeek the proposition \jev{} answers, with the same dialogue and records, for 599 LoCoMo and \lme{} questions (14,359 records, 1,004 of them gold evidence).
\jev{} separates gold evidence best (\cref{fig:system1}): AUC 0.942 overall against 0.900 for DeepSeek and 0.853 for gpt-4.1-mini, and 0.939 against 0.872 and 0.798 on \lme, whose chat-turn records rarely repeat the question's words.
A \jev{} call on 24 records takes 0.34\,s while answering two propositions per record; the LLMs take 1.05--3.74\,s for one (\cref{fig:system1}b).
System~1 work is thus done better, and 3--11 times faster, by a decision model: over the \workJevWaves{} waves of a question (\cref{tab:ours}), judging with an LLM would add seconds and cost while separating evidence less well.

\begin{figure}[!htb]
\centering
\includegraphics[width=\linewidth]{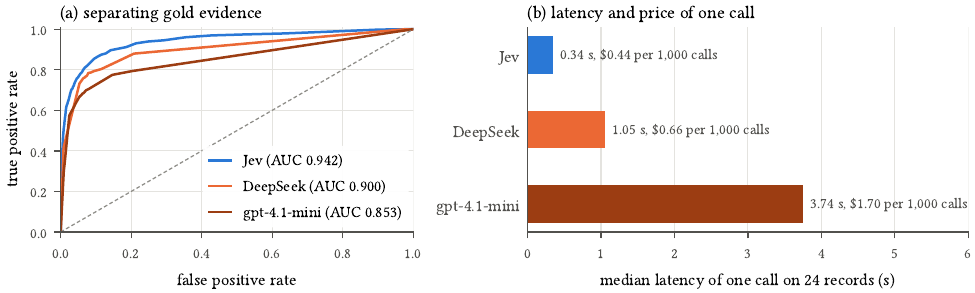}
\caption{\jev{} and two LLMs on the same 14,359 records: (a) ROC curves for separating gold evidence, pooled over records; (b) median latency and list price of one call on 24 records, in which \jev{} also answers whether each is no longer current.}
\label{fig:system1}
\end{figure}

\subsection{Judging Raw Records or Organizing Them}
\label{sec:jevmem}

The design can also be tested against a system that uses the same System~1 at write time.
Jev-Mem~\citep{jiang2026jevmem}, concurrent work, gives \jev{} the decisions of memory management: when a turn is written, \jev{} types it and judges its relations to earlier turns, building a graph with semantic, temporal, causal and entity edges; when a question arrives, \jev{} routes it over the graph, allocates the search budget and decides when to stop, and an LLM writes the answer.
We ran its released code (default profile, \jev{} 1.13 as in our runs) under our protocol: the same 1,540 LoCoMo questions, gpt-4.1-mini answering once per question from the question alone, and both graders.
Its released runner, by default, chooses among three answers by comparing them with the gold answer and picks its answer prompt by the question's category; we did neither.

Jev-Mem scores \jmAcc\% (\jmAccDS\% under DeepSeek), \jmDiff{} points below \sys{} (95\% CI \jmCI; \jmUp{} questions answered correctly only by \sys, \jmDown{} only by Jev-Mem), with the largest gaps on multi-hop (\jmMultiHop{} points) and temporal (\jmTemporal) questions, which need records connected or dated; given each question's category, it scores \jmLabels\%.
It sends less context (\jmCtx{} tokens; ECI \jmEci{} against \finEciLoCoMo) and does less work per question, one LLM call and \jmJevCalls{} \jev{} calls on average, but it spends \jmWriteRatio{} times as much per history at write time (\$\jmWrite{} against \$\writeLoCoMo), where \jev{} judges every turn against its candidates.
The two systems differ in more than where System~1 works, so this is not an ablation; but with the decision model held fixed, judging raw records once the question is known was more accurate than organizing them before it was.

\subsection{What Consolidation Adds}
\label{sec:consolidation}

Consolidation is the System~2 work that \sys{} moves off the read path, and answering the same questions with the index removed measures what it adds.
With gpt-4.1-mini answering it adds \consLoCoMoAbs{} points on LoCoMo, \consLMEAbs{} on \lme, \consBeamAbs{} on BEAM-100K and \consBeamTenAbs{} on BEAM-10M, and \dsConsLoCoMoAbs{} on LoCoMo and \dsConsLMEAbs{} on \lme{} with the reasoning model as System~2, each with a 95\% CI above zero; on HaluMem the two graders disagree in sign.
The gains fall where a question does not point to its evidence: instruction following, knowledge updates, summarization and abstention on BEAM; knowledge updates and multi-session questions on \lme; and multi-hop and temporal questions on LoCoMo, whose facts the timelines connect and date.
Consolidation multiplies the cost per question by \consCostRange, mostly for \jev{} to judge the index items.

%% file: sections/limitations.tex
\section{Limitations}
\label{sec:limitations}

\paragraph{Development and evaluation.}
\sys{} was developed on LoCoMo and \lme, and consolidation was designed after we had seen results on HaluMem and BEAM-100K, so none of the benchmarks is held out.
Our runs are graded by gpt-4.1-mini and DeepSeek and OmniMemEval's by gpt-4o-mini; graders differ by 1--2 points, and on HaluMem by more.
Each configuration ran once per setting, and two identical runs of an earlier configuration differed by up to 2.6 points.
We evaluate only conversational memory held in one Source; stores of documents, tasks or logs, which the design admits, are untested.

\paragraph{Dependence on \jev.}
\sys{} relies on \jev, a hosted, proprietary model; LLMs can stand in with lower discrimination at several times the latency and price, but no substitute was tested in the full system.

\paragraph{Latency and privacy.}
Our runs shared one laptop and public model APIs; a question took \workMeasured\,s end to end at the median, mostly in model calls.
Raw records keep everything an extractor would discard, and the index repeats some of it, so deletion and retention policies must cover both.

%% file: sections/conclusion.tex
\section{Conclusion}
\label{sec:conclusion}

\sys{} divides the work of memory between a slow System~2, an LLM that writes a few search queries, names what a reply needs and composes the answer, and a fast System~1, a decision model that makes many small, explicit judgments about the records those searches find; slow work that no reply can wait for, consolidating each record once, runs in the background.
Because nothing about a record is decided when it is written, the same agent can read any store that returns dated records, a general way to give agents a past beyond conversation.
With a small answering model, \sys{} is the most accurate of the systems in a public re-evaluation on LoCoMo and second on \lme, from under 4k tokens of context per question and with the lowest effective cost index on LoCoMo; with a reasoning model as System~2 it is on par with the best published results on \lme; and its cost per question stays nearly flat from 100K to 10M tokens of history.
Most of the read-time work of memory is System~1 work, which a decision model does better, faster and more cheaply than a general LLM; judging raw records once the question is known was also more accurate than using the same model to organize them in advance, and raw records let the same memory serve longer histories and stronger answering models without being rewritten.
Code, prompts, analysis scripts and run records (HaluMem's excepted) are available in the Mnemon repository~\citep{mnemon2026artifact}.